\documentclass{article}

\usepackage[margin=1.2in]{geometry}
\usepackage{times}
\usepackage{color}
\usepackage{amsmath}
\usepackage{amssymb}
\usepackage{amsfonts}
\usepackage{natbib}
\usepackage[normalem]{ulem}
\usepackage[titletoc,title]{appendix}
\usepackage{subcaption}
\usepackage{graphicx}
\usepackage{mathtools}
\usepackage{bm}
\usepackage[colorlinks,citecolor=green]{hyperref}
\usepackage{Definitions}
\usepackage{comment}

\allowdisplaybreaks

\usepackage{todonotes}
\newcommand{\kk}[2][]{\todo[inline,caption={}, linecolor=orange,backgroundcolor=orange!25,bordercolor=orange,#1]{\footnotesize KK: #2}}

\newcommand{\bc}{\bar c}

\DeclareMathAlphabet{\pazocal}{OMS}{zplm}{m}{n}

\newcommand{\bz}{\bar z}

\newcommand{\I}{I[c]}
\newcommand{\Ii}{I[c_i]}

\newcommand{\bI}{I[\bc]}
\newcommand{\bq}{\bar q}
\newcommand{\rmin}{r_{\min}}
\newcommand{\bu}{\bar u}

\newcommand{\tE}{\tilde E}

\DeclarePairedDelimiterX{\infdivx}[2]{(}{)}{%
  #1\;\delimsize\|\;#2%
}
\newcommand{\kl}{D_{\mathrm{KL}}\infdivx}

\usepackage[most]{tcolorbox}
\usepackage{pdfpages}

\usepackage{etoolbox}
\newtoggle{linjun}
 \togglefalse{linjun}
\newtoggle{draft}
\toggletrue{draft}

\newtoggle{haonan}
\toggletrue{haonan}

\begin{document}

\title{Can AI Be as Creative as Humans?}
\vspace{10pt}
\author{
\normalsize \textbf{Haonan Wang}\textsuperscript{1} \quad 
 \normalsize   \textbf{James Zou}\textsuperscript{2}\quad 
 \normalsize   \textbf{Michael Mozer}\textsuperscript{3}\quad 
  \normalsize  \textbf{Anirudh Goyal}\textsuperscript{3} \quad 
 \normalsize   \textbf{Alex Lamb}\textsuperscript{4}\quad  
 \normalsize   \textbf{Linjun Zhang}\textsuperscript{5}\quad \\  
  \normalsize  \textbf{Weijie J. Su}\textsuperscript{6}\quad  
  \normalsize  \textbf{Zhun Deng}\textsuperscript{7}\quad  
  \normalsize  \textbf{Michael Qizhe Xie}\textsuperscript{1}\quad  
 \normalsize   \textbf{Hannah Brown}\textsuperscript{1} \quad 
  \normalsize  \textbf{Kenji Kawaguchi}\textsuperscript{1}
\\[1.5mm]
  \small    \textsuperscript{1}National University of Singapore\quad  
 \small   \textsuperscript{2}Stanford University\quad 
 \small   \textsuperscript{3}Google DeepMind\quad \\ 
  \small  \textsuperscript{4}Microsoft Research\quad  
 \small   \textsuperscript{5}Rutgers University\quad  
  \small  \textsuperscript{6}University of Pennsylvania \quad  
  \small  \textsuperscript{7}Columbia University
}

\date{}
\maketitle
\vspace{-25pt}
\begin{center}
         \fontsize{9.5pt}{\baselineskip}\selectfont
         {Project Page:}~\tt\href{https://ai-relative-creativity.github.io/}{ai-relative-creativity.github.io}
         \vskip 0.25in
\end{center}

\begin{abstract}
Creativity serves as a cornerstone for societal progress and innovation. With the rise of advanced generative AI models capable of tasks once reserved for human creativity, the study of AI's creative potential becomes imperative for its responsible development and application.  
In this paper, we prove in theory that AI can be as creative as humans under the condition that it can properly fit the data generated by human creators. Therefore, the debate on AI's creativity is reduced into the question of its ability to fit a sufficient amount of data. 
To arrive at this conclusion, this paper first addresses the complexities in defining creativity by introducing a new concept called \textit{Relative Creativity}. 
Rather than attempting to define creativity universally, we shift the focus to whether AI can match the creative abilities of a hypothetical human.  
This perspective draws inspiration from the Turing Test, expanding upon it to address the challenges and subjectivities inherent in assessing creativity.  
The methodological shift leads to a statistically quantifiable assessment of AI's creativity, term \textit{Statistical Creativity}. 
This concept, statistically comparing the creative abilities of AI with those of specific human groups, facilitates theoretical exploration of AI's creative potential.
Our analysis of the AI training process reveals that by fitting extensive conditional data, including artworks along with their creation conditions and processes, without marginalizing out the generative conditions, AI can emerge as a hypothetical new creator.
The creator, though non-existent, possesses the same creative abilities on par with the human creators it was trained on. 
Building on theoretical findings, we discuss the application in prompt-conditioned autoregressive models, providing a practical means for evaluating creative abilities of generative AI models, such as Large Language Models (LLMs). 
Additionally, this study provides an actionable  training guideline, bridging the theoretical quantification of creativity with practical model training. 
Through these multifaceted contributions, the paper establishes a framework for understanding, evaluating and fostering creativity in AI models.
\end{abstract}

\vspace{-5pt}
\section{Introduction}
\label{sec:introduction}
\vspace{-2pt}

Creativity, usually deemed as a quintessential human trait, is not just an individual trait but a transformative force that shapes societies, catalyzing advancements in science, technology, and the arts. It forms the backbone of innovation, fuels economic growth, and  social change~\citep{amabile1996creativity, boden03, kirkpatrick2023can}.
In the rapidly evolving landscape of the digital age, artificial intelligence (AI) has introduced new avenues for creative endeavors. Advanced generative deep learning models have not only shown proficiency in solving complex problems, such as drug and protein synthesis~\citep{jumper2021highly}, but they have also excelled in artistic pursuits, including composing poetry and crafting narratives~\citep{franceschelli2023creativity}. Additionally, these models have displayed remarkable aptitude for generating novel ideas, even outpacing MBA students in terms of both quality and uniqueness of innovative product and service concepts~\citep{mba_vs_chatgpt}.
As AI's generative capabilities blur the lines between human and machine-generated work, this raises the stakes for the study of creativity, especially for the creativity of AI.

Traditionally, human creativity has been extensively studied and analyzed across various disciplines, such as psychology, philosophy, and cognitive science~\citep{amabile1996creativity, boden03, kirkpatrick2023can}. However, there is still no consensus on defining ``creativity'', primarily due to the subjective nature involved in the various definitions proposed in scholarly literature~\citep{runco2012standard,sawyer2012extending}. 
Even for the widely accepted definition of creativity as a blend of novelty and quality~\citep{boden03, camara2007creativity}, the inherent subjectivity of these criteria remains problematic. 
What is deemed novel and of quality can differ greatly across various cultures, disciplines, and time periods. For instance, a computer science paper written in iambic pentameter might be deemed highly novel but of low quality by one community, whereas another might view it as higher in quality but less novel. This lack of consensus on the creativity of humans hinders progress in understanding and developing creativity in AI.
In addition, sophisticated AI models are inherently designed to generate outputs reflecting their training data~\citep{foster2022generative}. This raises critical questions about the authenticity of their creations—whether they are genuinely original or merely repackaged elements~\citep{somepalli2023diffusion}. Such a tendency towards replication adds an additional layer of complexity to the endeavor of defining and analyzing creativity in AI. Against this backdrop, establishing a framework to understand AI's creativity is an essential step towards the responsible development and further evaluation of AI.

\begin{figure}[t]
    \centering
    \includegraphics[width=\linewidth]{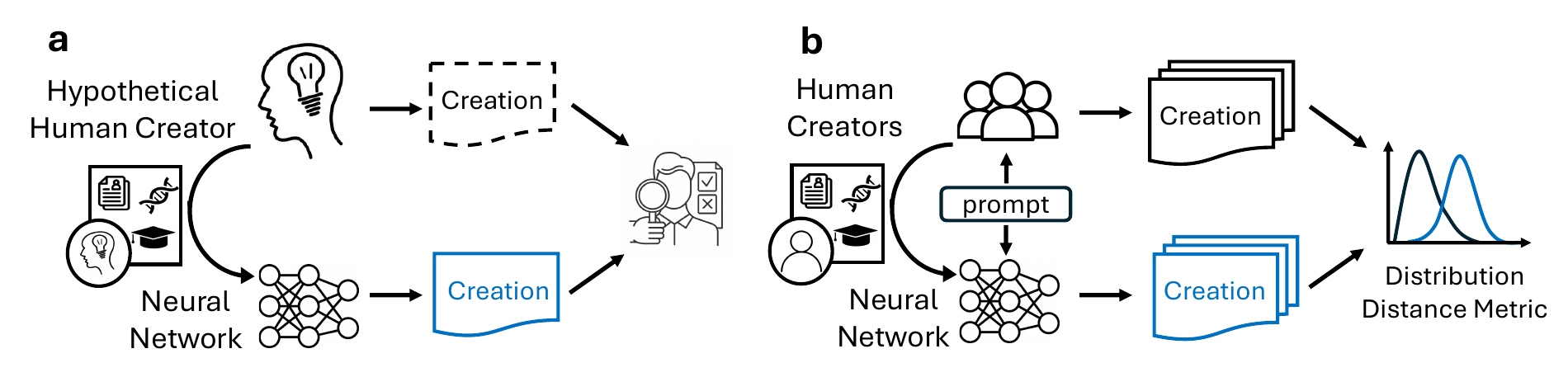}
    \caption{Illustration of Relative Creativity (\textbf{a}) and Statistical Creativity (\textbf{b}). (\textbf{a}): Relative Creativity assessed by AI's ability to create art indistinguishable from that of a hypothetical human, given the same biographical influences. (\textbf{b}): Statistical Creativity measured by AI’s ability to generate creations to prompts that are indistinguishable from those of existing human creators, as determined by a distribution distance metric.}
    \label{fig:illustration}
    \vspace{-10pt}
\end{figure}

In this work, we aim to establish a concrete framework for exploring creativity in artificial intelligence.  
We  begin with a  thought experiment: imagine an AI creates a piece of artwork, without prior knowledge of any similar creations. This AI-generated artwork remains undiscovered for 20 years. Decades later, when a human artist independently creates and gains recognition for an identical piece, a fascinating dilemma emerges. Should the AI's earlier work be recognized as creative, given it predates the human artist's creation, even before this new artist was known? The logical conclusion would be affirmative.
This thought experiment prompts us to shift from conventional approaches addressing the question, ``Can AI be creative?''—a question mired in the ambiguous and elusive task of defining creativity absolutely—we instead pivot towards a more concrete inquiry: \textit{``Can AI be as creative as humans?''}

In our study, we theoretically demonstrate that this is feasible if AI can properly fit data produced by human creatives.
In order to reach this conclusion, we formally structure the problem by introducing the concept of \textit{\textbf{Relative Creativity}}, where an AI model is deemed as creative as a hypothetical, yet realistic, human creator if it can produce works indistinguishable from that creator, as determined by an evaluator.
Building upon this, we present \textit{\textbf{Statistical Creativity}} , a means for understanding whether and to what degree AI achieves creativity by comparing it with existing human creators.
Through the lens of statistical creativity, our theoretical analysis of the training process suggests that AI has the potential to act as a hypothetical creator by effectively assimilating a massive amount of conditional data. 
This data encompasses a diverse range of creative works, along with the conditions and processes that led to their creation, from a specific group of creators. The creative ability of this hypothetical creator would be comparable to that of the group of creators upon which the AI's training is based.
The theoretical findings indicate that the emergence of human-like creativity will occur through the learning of extensive conditional data without marginalizing out the generative conditions and processes. 
In line with recent advancements in large language models (LLMs), such as GPT-4~\citep{openai2023gpt4} and Llama~\citep{touvron2023llama}, our research investigates statistical creativity in autoregressive models that excel in next-token prediction and generation within a prompting paradigm, which offer insights into the practical and valid evaluation for creativity of current AI models.
Moreover, our analysis of the training process provides practical guidelines for fostering creativity. It highlights the critical need to gather data on generative conditions and processes and to incorporate this information into the training process, contrasting with the current trend of prioritizing the accumulation of large volumes of creation data. 
Overall, our work offers a comprehensive framework for the theoretical exploration, practical evaluation, and enhancement of creativity in AI, effectively bridging theoretical analysis with practical application.

\section{Result Overview}
\label{sec:takeaway}
In an effort to establish a concrete framework for exploring creativity in artificial intelligence, Section \ref{sec:relative_creativity} introduces the concept of Relative Creativity (see Definition \ref{def:1}). This concept, mirroring the Turing Test's comparative approach to intelligence assessment, suggests that an AI model can be deemed as creative as a hypothetical human creator if its outputs are indistinguishable from those of the creator, as judged by a designated evaluator.
To facilitate the theoretical exploration of AI's creativity, Section \ref{sec:method} presents Statistical Creativity (Theorem \ref{thm:1}), which grounds the notion of relative creativity by comparing it to existing human creators.
Furthermore, the theoretical analysis of the training process in Section~\ref{sec:stat_relative_creativity_training} shows that AI has the potential for creativity comparable to humans, provided it can effectively assimilate the conditional data produced by human creators without neglecting the generative condition of creations (Corollary ~\ref{coro:4}). This theoretical result reduces the debate on AI's creativity to the question of collecting and fitting a massive amount of conditional data  with creative qualities.
Our exploration into the theoretical aspects of AI creativity yields valuable insights for practical model evaluation and training. In Section \ref{sec:stat_relative_creativity_autoregressive} and Section \ref{sec:stat_relative_creativity_prompt}, we study the form of statistical creativity in autoregressive models and prompt-conditioned autoregressive models. The derived concept of Prompt-Contextualized Autoregressive Statistical Creativity (Corollary \ref{coro:3}) provides insights into the practical evaluation of LLMs' creative capabilities.
Additionally, our theoretical findings yield actionable training guidelines that highlight the critical role of generation conditions of creations in fostering AI creativity (Remark~\ref{remark:loss}).
We strive for this paper not only to offer theoretical perspectives but to guide the discourse on AI creativity, advocating for relative evaluations to spur empirical research and establish a framework for assessing and fostering the creative capabilities of AI models.
We summarize our contributions as follows.
\begin{tcolorbox}[colback=orange!5!white,colframe=orange!75!black]
\begin{itemize}[leftmargin=1px]

\item 
\textbf{Theoretical Result on AI's Creativity:} 
We theoretically show that AI can be as creative as humans under the condition that AI can properly fit the sufficient amount of conditional data created by human creators. This implies that AI, through its ability to learn from historical data, has the potential to evolve into an entity capable of generating novel, creative outputs. These outputs, while not currently existing, are likely to be generated by a new human creator in the future.

\item 
\textbf{Introduction of Relative and Statistical Creativity:} 
We present the concept of Relative Creativity, a notion for understanding the creative capabilities of artificial intelligence. This concept is defined from a comparative perspective, eschewing absolute definitions in favor of a relative understanding of creativity. To further facilitate theoretical explorations, we introduce Statistical Creativity, a means for assessing whether and to what extent an AI model achieves creativity by comparing its ability to  those of a specific human population.

\item 
\textbf{Integration of Subjectivity:} 
The introduced notions of creativity acknowledges the inherent subjectivity in creativity and incorporates the subjectivity into the comparison process, akin to how the Turing Test assesses intelligence through comparison rather than fixed definitions. The subjective nature of creativity is crystallized in the choice of the anchor population, allowing the study of AI creativity to maintain a degree of objectivity.

\end{itemize}
\end{tcolorbox}

\begin{tcolorbox}[colback=orange!5!white,colframe=orange!75!black]
\begin{itemize}[leftmargin=1px]
\item 
\textbf{Creativity Measurement of Contemporary AI Models:} Delving into autoregressive models, we propose a practical measure of statistical creativity for autoregressive models with next-token prediction. This measure can be applied to generative models with the prompting setup, ensuring its applicability to cutting-edge Large Language Models (LLMs).

\item 
\textbf{Practical Training Guidelines:} 
Theoretical analysis of the training procedure emphasizes the importance of collecting the generative conditions of creation data, in contrast to the prevalent approach of simply accumulating large datasets. And it suggests that fitting a substantial amount of conditional data, without marginalizing out the generative conditions, is crucial for the emergence of creative abilities.
\end{itemize}
\end{tcolorbox}
\section{Relative Creativity}
\label{sec:relative_creativity}

\subsection{What is Relative Creativity?}
\label{sec:what_is_relative_creativity}
Diverging from traditional methods setting an absolute threshold or checklist to determine if AI can be deemed creative, that necessitates a foray into the contentious task of defining creativity universally, we shift the study by proposing an creativity notion for AI called \textit{\textbf{Relative Creativity}}. Relative creativity sidesteps the difficulties associated with defining creativity in absolute terms, mirroring the Turing Test's approach to intelligence evaluation~\citep{turing2009computing}. Note, the Turing Test eschews absolute definitions of intelligence, instead opting for a relative metric that contrasts machine behavior with human responses in conversational scenarios.  
In the assessment of creativity, an AI model is deemed ``relatively creative'' if it can generate creations that are indistinguishable from those of a hypothetical yet plausible human creator, as judged by an evaluator. 
The way in which this notion of creativity is ``relative'' is that it depends on the individual to whom the entity is being compared. 
For instance, if we regard 19th-century painting artists as creative, then we might assume that the hypothetical creator is derived from the distribution of such artists. If the creations generated by the AI, when conditioned on the biographical data or characteristics of this hypothetical creator, are indistinguishable from what that human would have created, then we can conclude that the AI model is as creative as the painting artists from the 19th century.
Relative creativity distinguishes itself through acknowledges the inherently subjective facets of creativity—such as originality, divergent thinking, and problem-solving skills—these elements are integrated into the anchor selection process. Because the subjectivity of creativity is crystallized into the choice of the anchor human, relative creativity leaves the study of AI creativity to be objective.
Regardless of whether the creative outputs of a specific individual or group satisfy the varied, subjective standards for what constitutes creativity, this notion reframes the debate, enabling a focused, empirical investigation into whether an AI model can replicate the creative capacities of a predefined human benchmark.

\subsection{Notation and Formal Definition}
\label{sec:relative_creativity_definition}
\noindent\textbf{Preliminary Notations.}
Imagine if the task is to write poems.
Let $\Xcal$ represent a finite set of poems, denoted by $|\Xcal|<\infty$, where each individual poem in the set is given as $x \in \Xcal$. Consider a generative model $q$ that receives a specific form of information $I\in\mathcal I$. We assume that $I$ consists of partial information pertinent to poets, such as personal background, artistic education, societal and historical Context, etc.
The model outputs a creation according to the conditional distribution $q(\cdot | I)$. This format of conditional generation is widely adapted by current practical models~\citep{rombach2021highresolution, brown2020language}. let $C$ be a set containing complete information about poets, with each poet's data represented by $c \in C$. While $c$ offers a holistic view of a creator, $\I$ only provides a subset to prevent the AI from merely duplicating creations from $c$. Notably, $c$ includes poems linked to the given poem, whereas $\I$ purposely omits these. For example, $\I$ might include AI-generated or AI-engineer-designed synthetic data, such as imagined profiles, background details, simulated upbringing environments, cognitive patterns and experiences.

\noindent\textbf{Probability Distribution and Evaluator Function.}
Define $\Dcal_{C}$ as a probability distribution over $C$. This set, $C$, is versatile, covering all potential creators, whether they exist now, in the past, or in the future. An evaluation function $L$ determines if the generative model $q$ mirrors a particular human creator defined by the information $c$. Specifically, $L(q, c) = 0$ indicates successful emulation, while $L(q, c) = 1$ denotes a failure.

\noindent\textbf{Defining AI Creativity.}
Let $\bc \sim \Dcal_{C}$ stand for a hypothetical creator, one not grounded in reality, from the creator distribution. Here, $\Mcal(\bc)=q(\cdot\mid \bI)$. Given these components, we outline AI creativity as the AI model's capability to emulate a new, plausible, yet non-existent human creator by given biography of the virtual creator, evaluated by $L$. This concept is formalized in the subsequent definition of \textit{relative creativity}:

\begin{definition}[\textbf{Relative Creativity}] \label{def:1}
An AI model, denoted as $\mathcal{M}$, achieves $\delta$-creativity (with respect to evaluator $L$ under creator distribution $\mathcal{D}_{C}$)  if it is indistinguishable from a plausible, yet non-existent human creator to the degree where $L(\Mcal(\bc),\bc)=0$ with a probability of at least $1 - \delta$ for $\bc \sim \Dcal_{C}$.
\end{definition}

\begin{remark} \label{remark:1}
Relative creativity compares the creative abilities of AI and a hypothetical human creator by scrutinizing the creations of the creator and the AI model's results—derived from reasoning over a human creator's biography. In doing so, we leverage the same informational foundation—the human's life history and personalized knowledge—to construct a direct comparison of creativity between the AI and the hypothetical human. Relative creativity acknowledges the inherently subjective nature of creativity, which encompasses aspects like originality, divergent thinking, and problem-solving skills. These subjective elements are considered in the process of choosing the human benchmark for comparison.
\end{remark}

\section{Statistical Creativity}
\label{sec:method}
In this section, we introduce the Statistical Creativity to enable the theoretical exploration of AI's creativity. This approach assesses whether and to what extent a model achieves creativity, basing the comparative assessment on observable human creators.

\subsection{Definition of Statistical Creativity}
\label{sec:stat_relative_creativity}
Due to the agnostic nature of the authentic distribution of hypothetical human creators, in this section, we consider Statistical Creativity, an approach designed to assess the indistinguishability between the creative abilities of an AI model and those of observable human creators. The human creators are represented as $(c_i)_{i=1}^n \sim (\Dcal_{C})^{\otimes n}$, $c_i$ is independent and identically sampled from creator distribution $\Dcal_{C}$. Formally, the term $ E_{0}(q)$ is defined as:
\begin{align}
E_{0}(q)=\frac{1}{n}\sum_{i=1}^n L(q(\cdot\mid \Ii),c_i).
\end{align}
A low $ E_0(q) $ value for a generative model signifies its high resemblance to the majority of creators  $ c_i $. In these instances, it becomes challenging for the evaluator to distinguish between the outputs of the model $q(\cdot\mid \Ii) $ and the human creators $ c_i $.

To facilitate the adjustment from theoretical creators (as defined in Definition~\ref{def:1}) to actual creators, whose data is readily obtainable, we present the Theorem \ref{thm:1}. This theorem explicitly outlines the conditions under which an AI model can be deemed as exhibiting $\delta$-creativity in relation to a specific group of human creators, as assessed by evaluator $L$.
Specifically, if the term $ E_{0}(q) < \delta $ when evaluated over a sufficiently large sample set $\left( n \ge\frac{\ln(1/\delta')}{2(\delta-E_{0})^{2}} \right)$, then the AI can be regarded as $\delta$-creative as humans from $\mathcal{D}_{C}$.
\begin{theorem}[\textbf{Statistical Creativity}] \label{thm:1}
Suppose we have a positive integer $n\in \NN_+$, and positive real numbers $\delta,t>0$. Let $\Mcal(c)=q(\cdot\mid I[c])$ be an AI model. If this model satisfies $E_{0}(q)<\delta$ and $n \ge\frac{\ln(1/t)}{2(\delta-E_{0})^{2}}$, then the AI model $\Mcal$  is $\delta$-creativity (w.r.t. $L$ under $\Dcal_{C}$), with probability at least $1-t$ over the draw of $(c_i)_{i=1}^n \sim (\Dcal_{C})^{\otimes n}$.
\end{theorem} 

\begin{proof}
The proof is presented in Appendix \ref{app:proof}.
\end{proof}


\begin{remark} 
In this theorem, we theoretically demonstrate that AI can be as ($\delta$-)creative as humans, provided it can assimilate a specified amount of data ($n$) from human creators.
Additionally, moving beyond a ``creative or not'' categorization, statistical creativity presents a sophisticated perspective where creativity is characterized by its deviation, $\delta$,  to a pre-selected existing human creators as benchmarks. 
Crucially, statistical creativity does not require an AI to perfectly replicate human creators. Instead, it highlights the importance of achieving a certain level of similarity that is appreciable through the lens of the evaluator, represented by $L$. 
When a human assumes the role of the evaluator, a creatively successful AI is expected to as skillful as a novel creator, as perceived from the human evaluator's perspective.
\end{remark}

\subsection{Statistical Creativity for Autoregressive Model}
\label{sec:stat_relative_creativity_autoregressive}
In this section, we delve into applying statistical creativity into autoregressive models, which are fundamental to the contemporary generative AI models. We explore the specific form of statistical creativity in this context, thereby enhancing our understanding of it within next-token prediction mechanisms.

Our proposition is straightforward: if an AI can generate sequences (like poems or stories) that mirror the works of a group of human artists, it demonstrates a level of creativity comparable to that group. To quantify this assessment, we introduce the metric $E_1(q)$. This metric estimates the indistinguishability between the creative abilities of an autoregressive model and those of human creators by measuring the log-likelihood of next-token predictions across diverse pairs of creators and their creations.
\begin{align}
E_1(q)=-\frac{1}{n}\sum_{i=1}^n \frac{1}{r(c_{i})} \sum_{t=1}^T \log q(x_{i} ^{(t)}\mid x^{(t-\omega:t-1)}_{i},I[c_{i}]).
\end{align}
In this formula, the pairs $(x_i,c_i) _{i=1}^{n}$ are drawn independently and uniformly from $\Dcal$, wherein $\Dcal(x,c) = \Dcal_C(c)\cdot p(x\mid c)$. 
The term  $r(c_i)= \tau+H[ p(\cdot\mid c_{i})]$ is of interest, where $p(\cdot\mid c_{i})$ delineates the authentic (yet unknown) distribution of creations for the creator $c_{i}$. And $H[ p(\cdot\mid c_{i})]$ signifies the entropy of $p(\cdot\mid c_{i})$. The positive constant $\tau$ sets a threshold, details of which will be delved into later. 
Breaking it down further, the creation $x$ consists of $T$ tokens. Define $T\ge 1$ and represent $x ={\{x^{(t)}\}}_{t=1}^T$. The term $q (x \mid  I[c_i])=\prod_{t=1}^T q(x^{(t)} \mid x^{(t-\omega:t-1)},I[c_i])$ is predicated on a context window size of $\omega\ge 0$. For specific conditions where $t=1$ or $\omega=0$, the expression simplifies to $q(x^{(t)} \mid I[c_i])$. Likewise, for $t<\omega$, the notation changes to ${(t-\omega:t-1)} \triangleq {(1:t-1)}$. 

\begin{remark}
It's worth noting that by setting $T=1$, the non-autoregressive structure is retained, preserving the metric's compatibility. In addition, an intriguing aspect of $E_1(q)$ is the term $\frac{1}{r(c_{i})}$. This term plays a pivotal role in ensuring that if the entropy in the creation process by a creator $c_i$ is vast, the corresponding log-likelihood receives lesser weight. This adjustment effectively captures the inherent unpredictability and diversity characterizing human creativity.
\end{remark}

Intuitively, the evaluator $L$ is unable to discern between an AI creator $q$ and an actual creator $p$ if the KL divergence between the two remains minimal. To crystallize this concept, we frame it as the following assumption:
\begin{assumption} \label{assumption:1}
There exists a positive threshold $\tau$ such that $L(q(\cdot\mid \I),c)=0$ whenever $\kl{p(\cdot\mid c)}{q(\cdot\mid \I)} < \tau$, for any $c \in C$. 
\end{assumption}
Within this context, $\tau>0$ stands as the threshold fulfilling the condition stipulated in Assumption \ref{assumption:1}. Furthermore, we introduce $\rmin \in \mathbb{R}^{+}$ as the least positive real number ensuring $\rmin\le r(c)$ with near certainty for $c\sim \Dcal_C$. 
The above assumption delineates the conditions under which an AI is deemed statistically creative. With the condition, the following theorem presents the statistical creativity in the next-token prediction setup.
\begin{theorem}[\textbf{Autoregressive Statistical Creativity}] \label{thm:2}
Provided Assumption \ref{assumption:1} remains valid, and for a  model $\Mcal$ such that its negative log likelihood $ -\log q(x\mid I[c])\le M$ nearly always over $ (x,c)\sim \Dcal$, where $M\in \mathbb{R}^{+}$. If $n\in \NN_+$, $\delta,t>0$, and $\Mcal$ is an AI model where $E_{1}<\delta$ and $n \ge\frac{M^{2}\ln(1/t)}{2\rmin^{2}(\delta-E_{1})^{2}}$, then the model $\Mcal$ is deemed $\delta$-creativity (with respect to $\mathcal{D}_{C}$), with a probability at least $1-t$ across the sampling of $ (x_i,c_i) _{i=1}^{n}\stackrel{i.i.d.}{\sim}\Dcal$.
\end{theorem} 
\begin{proof}
The proof is presented in Appendix \ref{app:proof}.
\end{proof}

\begin{remark} \label{remark:them2}
Contrasting it with Theorem \ref{thm:1}, Theorem \ref{thm:2} reveals that demonstrating creativity does not necessarily require strictly imitating each creator. Instead, it emphasizes the significance of the log-likelihood in next-token prediction across a wide range of creators. Furthermore, Theorem \ref{thm:2} addresses a critical question: Is it essential to have a large number of samples from each creator to demonstrate a model's statistical creativity? This inquiry is vital, considering the potential bottleneck in creative AI development due to the extensive collection of each creator's works. This theorem provides insight into this issue, indicating that amassing a vast dataset for each creator is not as crucial as previously thought. What is more important, as Theorem \ref{thm:2} highlights, is the creator-creation pairs (i.e., $n$ in $\{(x_i,c_i)\} _{i=1}^n$). This implies that with a sufficiently large number of creators, the sample size per creator can be relatively small, thereby alleviating data collection challenges in developing creative AI models.
\end{remark}

\subsection{Statistical Creativity for Large Language Model}
\label{sec:stat_relative_creativity_prompt} 
The technique of prompting serves as a powerful tool for harnessing the inherent capabilities of generative models, particularly Large Language Models (LLMs)~\citep{bommasani2022opportunities}, which are recognized as exhibiting a certain level of creative potential~\citep{Zhao2023ASO}.
In modeling this prompting setup, we denote $U$ as a set encompassing various contexts of creations. Each context within this set is denoted as $ u $. We define $ \Dcal_U $ as the probability distribution over $U$.
Define $ \Dcal_{U,C}(u,c) = \Dcal_C(c)\Dcal_U(u) $, where $ \{u_i\}_{i=1}^n \stackrel{i.i.d.}{\sim} (\Dcal_U)^{\otimes n} $ symbolizes a sequence of existing context prompts, assumed to be independently and identically distributed. The notation $ \bu \sim \Dcal_U $ corresponds to a possible new context prompt.
To clarify the distinction between the variables $ c $ and $ u $, consider a practical example involving GPT-4~\citep{openai2023gpt4}. Here, the creator's information, such as a biography, is represented by $ c $ and serves as the system prompt, while $ u $ corresponds to the user's input prompt.

To facilitate the measurement of statistical creativity in cutting-edge models, we expand our previously outlined functions to incorporate the variable $ u $. 
Consequently, functions such as $L(\bq,c) $ and $q(x| \I) $ are updated to $L(\bq,u,c) $ and $q(x| u,\I) $, respectively. This modification is consistently applied across related definitions. For instance, $\Mcal(\bc) $ is now redefined as $\Mcal(\bz)=q(\cdot| \bu, \bI) $, where $\bz=(\bu, \bc) $ represents the combined user and system prompts. Building on these modifications, we refine the definition of statistical creativity as follows:

\begin{definition} 
An AI generative model, denoted as $\Mcal$, is termed $\delta$-creativity (w.r.t.  $L$ under $\Dcal_{U,C}$), if it behaves in a manner indistinguishable from a hypothetical (yet plausible) human creator when faced with new context prompts. This is characterized by the condition $ L(\Mcal(\bz),\bz) = 0 $ with a probability of at least $1-\delta$ over the sampling of $(\bu ,\bc )\sim \Dcal_{U,C}$. 
\end{definition}

The aforementioned definition expands upon Definition~\ref{def:1} by integrating it into a prompting setup. Subsequently, we offer results that are analogous to those discussed in Section~\ref{sec:method}. We introduce the metric,
\begin{align}
\quad\quad  E_{2} = \frac{1}{n} \sum_{i=1}^n L(\Mcal(z_{i}), z_{i}), \quad\quad\text{where}\quad (z_{i})_{i=1}^n \sim (\Dcal_{U,C})^{\otimes n}. 
\end{align}

\begin{corollary}[\textbf{Prompt-Contextualized Statistical Creativity}] \label{coro:2}
For a given positive integer $n\in \NN_+$,  $\delta,t>0$, and constants $\delta,t>0$, suppose $\Mcal$ is an AI model satisfying $E_{2}<\delta$ and 
$ n \ge \frac{\ln(1/t)}{2(\delta-E_{2})^{2}}. $
Then, the AI model $\Mcal$ is  $\delta$-creativity under $\Dcal_{U,C}$ with a probability not less than $1-t$ over the sampling of $(z_{i})_{i=1}^n \sim (\Dcal_{U,C})^{\otimes n}$.
\end{corollary}
\begin{proof}
The proof is presented in Appendix \ref{app:proof}.
\end{proof}

\begin{assumption} \label{assumption:2}
There exists $p$ and $\tau>0$ such that $L(q(\cdot\mid u,\I),u,c)=0$ if $\kl{p(\cdot\mid u,c)}{q(\cdot\mid u,\I)} < \tau$, applicable for any $c \in C$ and $u \in U$. 
\end{assumption}
Correspondingly, let's define the following metirc $E_3$:
\begin{align}
    E_3=-\frac{1}{n}\sum_{i=1}^n \frac{1}{r(u_{i},c_{i})} \sum_{t=1}^T \log q(x_{i} ^{(t)}\mid x^{(t-\omega:t-1)}_{i},u_{i},I[c_{i}]),
\end{align}
where the triple  $(x_i,u_{i}, c_i) _{i=1}^{n}\sim\Dcal^{\otimes n}$ with $\Dcal(x,u,c)=  \Dcal_{U,C}(u,c)p(x\mid u,c)$. 
Define $r(u,c)= \tau+H[ p(\cdot\mid u,c_{})]$. Let $\rmin$ be a positive real number such that  $\rmin\le r(u,c)$ almost surely over  $(u,c)\sim \Dcal_{U,C}$. Then, we have the following corollary for prompt-contextualized autoregressive models.

\begin{corollary}[\textbf{Prompt-Contextualized Autoregressive Statistical Creativity}] \label{coro:3}
Let Assumption \ref{assumption:2} hold. Let $q$ be given such that $-\log q(x\mid u,I[c])\le M$ almost surely over  $(x,u,c)\sim \Dcal$ for some  $M>0$. Let $n\in \NN_+$,  $\delta,t>0$, and $\Mcal$ be an AI model such that $E_{3}<\delta$ and   $n \ge\frac{M^{2}\ln(1/t)}{2\rmin^{2}(\delta-E_{3})^{2}} $. Then, the AI model $\Mcal$  is $\delta$-creativity , with probability at least $1-t$ over the draw of    $(x_{i},u_{i},c_i)_{i=1}^n \sim \Dcal^{\otimes n}$.
\end{corollary}
\begin{proof}
The proof is presented in Appendix \ref{app:proof}.
\end{proof}

\begin{remark}
The corollaries presented above expand upon our previous concept of statistical creativity, illustrating the ways in which this form of creativity can be extended and measured in contemporary autoregressive models with a prompting setup, such as large language models (LLMs). Additionally, these corollaries provide guidelines on the data format, specifically the triple $(x_{i},u_{i},c_i)$, and the necessary volume of data for practical evaluations.
\end{remark}

\section{Emergence of Relative Creativity} 
\label{sec:stat_relative_creativity_training}
The discussions in the preceding sections have concentrated on establishing a framework to define and assess the creativity of AI models. This introduces a critical question: Can AI models attain a level of creativity comparable to humans through training, and if so, how? This section seeks to theoretically elucidate this question, leading to our principal conclusion: AI can be as creative as humans under the condition that it can properly fit a sufficient amount of conditional data generated by human creators.

\subsection{The Loss Preserving Generative Conditions}
We revisit and scrutinize the terms previously introduced in the context of statistical creativity, particularly focusing on terms like $E_2$ and $E_3$. The derivable nature of these terms suggests a pathway towards defining  \textit{Statistical Creativity Loss}, Equation~\eqref{eq:scl}.
The conditional loss underscores the importance of preserving the generative conditions, denoted by $\bu$ and $\bI$, of the creations.
\begin{align} \label{eq:scl}
L_{\bar z}[q]=L(q(\cdot\mid \bu, \bI), \bz)
\end{align}
For the notation, we follow the setup of Section~\ref{sec:stat_relative_creativity_prompt}, where $\bz=(\bu,\bc)$ encapsulating the combined user and system prompts.  We consider a learning algorithm $\Acal$ and denote its output as $q_{S}=\Acal(S)$, given the sample set $S$. This is established under the condition that $-\log q_{S}(x\mid u,I[c])\le M$ almost surely over $v \sim \Dcal$, where $S=(v_{i})_{i=1}^n \sim \Dcal^{\otimes n}$. Besides, for the following discussion, we define $\psi(q,v)=-\frac{1}{r(z)} \log q(x\mid u,I[c])$ where $v=(x,u,c)$.

\begin{remark}
In contrast to the prevalent training paradigm, which emphasizes accumulating increasingly large datasets for model pretraining, the loss function described in Equation~\eqref{eq:scl} explicitly involves the generative processes and conditions.
\end{remark}

\subsection{From Generalization to Emergence of Creativity}
A pivotal aspect of our discussion is centered on the concept of generalization in deep learning~\citep{kawaguchi2022generalization}. We introduce a placeholder notation $Q(t)$ to denote for potential overfitting of the model to its training dataset. The choice to employ a placeholder is strategic, given that the exact metrics and quantifications of overfitting—captured by $Q(t)$—are at the forefront of ongoing research in the realm of generalization bounds.
Note, this modular approach allows for easy integration of future advancements in the study of generalization into our framework, by simply updating the placeholder, $Q(t)$. For clarity, we will also present some potential manifestations of $Q(t)$ based on the prevailing understanding from the study of deep learning generalization.
Concretely, for any $t>0$, we define $Q(t)$ such that with probability at least $1-t$ over an draw of $S=(v_{i})_{i=1}^n \sim \Dcal^{\otimes n}$, the following inequality holds:
\begin{align*}
\EE_{v}[\psi(q_{S},v)]- \frac{1}{m}\sum_{i=1}^{m} \psi(q_{S}, v_{i})
\le \tilde \Ocal\left( \sqrt{\frac{Q(t)}{n}}\right).
\end{align*}

For deep neural networks, several foundational works have explored their complexity. For instance, with the Rademacher complexity bound   \citep{bartlett2002rademacher,mohri2012foundations,shalev2014understanding}, the size-independent sample complexity of neural networks \citep{golowich2018size}  provides 
$$
Q(t)=B^{2} \rho \prod_{j=1}^\rho M_F^2(j)+\ln(1/t),
$$ 
where $\rho$ is the number of layers of the network, $B$ is the upper bound on the Euclidean norm of the network, and $M_F(j)$ is the upper bound on the Frobenius norm of the weight matrix at $j$-th layer. 

In addition, with the recent result in information theory \citep{icml2023kzxinfodl}, we have,
$$
Q(t)=\min_{j \in \{1,\dots,\rho\}} I(X;Z_j|Y)+I(S; \Theta_j)+\ln(1/t),
$$
where $I(X;Z|Y)$ is the conditional mutual information between the network input $X$ and the output  $Z_j$  of $j$-th  layer, conditioned on the target label $Y$. Here, $I(S;\phi_j)$ is the mutual information between $S$ and the set of weight parameters from the first layer to    $j$-th  layer.

Besides, the sample complexity based on robustness \citep{xu2012robustness,kawaguchi2022robustness} indicates,
$$
Q(t)=c(S)^{2}+\Ncal((2\sqrt{n})^{-1},\Vcal,\kappa)+\ln(1/t),
$$
where $\Vcal$ designates a chosen compact space of $v$ relative to metric $\kappa$. The constant $c(S)$ is Lipschitz and defined as: $|\psi(q_{S}, v) - \psi(q_{S}, v')| \le c(S) \kappa(v,v')$ for any pair $(v,v') \in \Vcal$. Notably, $\Ncal((2\sqrt{n})^{-1},\Vcal,\kappa)$ is the $\epsilon$-covering number of $\Vcal$ (as detailed in Definition 1 of \citealp{xu2012robustness}).

Moreover, for  any model configuration, if we chose $\Acal: \Scal \rightarrow \Qcal$ with $|\Qcal| < \infty$, an use of  concentration inequalities (e.g.,  \citealp{kawaguchi2022generalization}) provides
$$
Q(t) = \ln\left(\frac{|\Qcal|}{t}\right).
$$ 

Different from the traditional discussions on model generalization, which often focus on metrics like classification accuracy, Corollary \ref{coro:4} studies the generalization of models training with the statistical creativity loss, Equation~\eqref{eq:scl}. 
\begin{corollary} \label{coro:4}
Let Assumption \ref{assumption:2} hold. Then,
with probability at least $1-\delta$ over  an draw of $S=(v_{i})_{i=1}^n \sim \Dcal^{\otimes n }$ and $\bz\sim \Dcal_{U,C}$, 
\begin{align} \label{eq:1}
L_{\bar z}[q_{S}] <2\delta^{-1}E[q_{S}]+\Ocal\left( \sqrt{\frac{\delta ^{-2}Q(\delta/2)}{n}}\right)\le\frac{2}{\delta\rmin}\tE[q_{S}]+\Ocal\left( \sqrt{\frac{\delta ^{-2}Q(\delta/2)}{n}}\right) .
\end{align}
\end{corollary} 
\begin{proof}
The proof is presented in Appendix \ref{app:proof}.
\end{proof}

\begin{remark}
    Corollary \ref{coro:4} outlines the upper bound of statistical creativity for a model trained with the conditional loss, as described in Equation~\ref{eq:scl}, over the sample set $S$. This shifts the discussion of AI's creativity to focus on its capacity to adequately fit a substantial amount of conditional data from existing human creators. Additionally, it suggests that if AI can effectively assimilate a massive amount of conditional data (without neglecting the generative process), it could act as a novel hypothetical creator, matching the quality of the group of creators upon which it was trained.

These findings point to a new direction in AI creativity research. Previous efforts have concentrated on compiling large datasets in the hope that this would organically lead to creative abilities in AI models. However, our research underscores a critical subtlety: the significance of not overlooking the generative condition/process inherent in these creative works. By integrating and comprehending the generative process within the data fitting process, AI could potentially achieve a level of creativity comparable to humans.
\label{remark:loss}
\end{remark}

\section{Related Works}
\label{sec:related}


\subsection{Definitions of Creativity}
Research into creativity has spanned several decades, with the topic remaining a central point of interest in psychology, cognitive science, and philosophy. While there has been an abundance of research on the subject, there is no universally agreed upon definition. Indeed, literature showcases an array of over a hundred proposed definitions, underscoring the multifaceted essence of creativity \citep{aleinikov00,treffinger96,boden03, elgammal_can_2017}. Within the field of computational creativity, a commonly accepted framework for assessing creativity uses the ``four P's'': namely, \textit{person} (the creator of a work), \textit{press} (the environmental context for a work), \textit{process} (how a work is created), and \textit{product} (the work itself) \citep{jordanous16}. \cite{boden03} presents a triadic criterion for gauging machine creativity, highlighting artifacts or ideas that are ``new, surprising, and valuable.'' Building upon this foundation, \cite{boden03} identifies three nuanced forms of creativity: Combinatorial, Exploratory, and Transformational. This classification melds the \textit{process} and \textit{product} dimensions of the four P's in order to focus on the production of surprising or novel outputs. The previous works consider \textit{absolute} creativity at the level of each creation given each environment without mathematical concreteness. We depart from previous work here. Instead of considering \textit{absolute} creativity at the creation level, we consider the \textit{relative} creativity at the meta-level of parallel universes consisting of humans and AI algorithms. This allows us to avoid  the question of what is creative in an absolute sense, and thus enables us to rigorously define objectives and desired properties via concrete mathematical formula instead of English descriptions.



\subsection{Applications in Vision and Language}
The field of creative image generative models has seen significant growth, raising questions regarding machines' ability to produce creative art. \cite{hertzmann2018can} delved into this, highlighting intersections between computer graphics and artistic innovation. \cite{xu2012fit} introduced creative 3D modeling that aligns with user preferences while ensuring variety. Generative Adversarial Networks (GANs)~\citep{goodfellow2014generative}, used by \cite{elgammal_can_2017}, drive the creation of distinctive artistic styles by maximizing deviations from known styles. \cite{sbai2018design} further this deviation by encouraging models to differ from training set styles. The perspective of creative generation as composition is evident in works by \cite{ge2021creative} and \cite{ranaweera2016exquimo}, emphasizing the integration of detailed elements. \cite{vinker2023concept} build on this by fragmenting personalized concepts into visual elements for innovative reassembly, enriching creative output.

Parallel to advancements in vision, the evolution of language models has spurred inquiries into optimizing data usage to bolster their adaptability across various domains, tasks, and languages~\citep{gururangan2020don,devlin-etal-2019-bert,conneau-etal-2020-unsupervised}. 
Researchers have shown an inclination to deploy language models in deciphering nuances in human communication~\citep{schwartz_personality_2013,wu2022using}. This insight is further harnessed to refine classification models~\citep{hovy_demographic_2015,flek_returning_2020}.
With the increasing popularity of generative models, there has also been an interest in controllable text generation, in which a model's output must satisfy constraints such as politeness~\citep{saha-2022-countergedi,sennrich-etal-2016-controlling}, sentiment~\citep{liu2021mitigating,dathathri2019plug,he2020probabilistic}, or other stylistic constraints. Finally, text style transfer (TST), in which the goal is to transform the style of an input text to a set goal style, has also become a popular task. Style may refer to a range of different text and author specific features including politeness~\citep{madaan-etal-2020-politeness}, formality~\citep{rao-tetreault-2018-dear,briakou-etal-2021-ola}, simplicity~\citep{zhu-etal-2010-monolingual,van_den_bercken-2019-evaluating,weng-2019-unsupervised,cao-etal-2020-expertise}, author~\citep{xu2012paraphrasing,carlson2018evaluating}, author gender~\citep{prabhumoye-etal-2018-style}, and more~\citep{jin-2022-deep}. 
While all these applications seek to apply elements of creativity in generative models, none directly defines creativity or seeks to directly optimize for it. Instead, the focus is on improving models' ability to excel at predefined tasks as proxies for creativity. In contrast, our research pivots on establishing a theoretical foundation for creativity. This framework of creativity naturally includes preceding insights on diversity and the quality of generation. 
We anticipate that our contributions will lay the groundwork for future endeavors, guiding the enhancement of model creativity.

\section{Conclusion}
\label{sec:conclusion}
In this paper, we embarked on an exploration of creativity within the realm of artificial intelligence. We theoretically proved that AI can be as creative as humans, provided it can effectively fit conditional data, including artworks and the corresponding creation conditions and processes, generated by human creators.
Our journey began with the introduction of Relative Creativity, a paradigm-shifting concept that defines the creativity in AI from a comparative standpoint, moving away from absolute standards. This approach acknowledges the inherent subjectivity in creative processes and incorporates it into the selection of comparison anchors, drawing inspiration from the Turing Test’s method for assessing intelligence.
We then introduced Statistical Creativity, focusing on whether AI can replicate the creative outputs of specific observable human groups. This concept not only facilitates the theoretical analysis of AI's creativity but also provides a basis for assessment, thus bolstering the practical applicability of our theoretical results.
On the practical front, aligned with the recent advancements in generative AI models, we have further investigated statistical creativity within the context of autoregressive models with prompting. Our findings reveal practical way to measure the statistical creativity of large language models (LLMs).
From a theoretical viewpoint, through the lens of statistical creativity, our analysis of the AI's training process indicates that by fitting extensive conditional data without marginalizing out the generative conditions, AI can emerge as a hypothetical new creator. This creator, though non-existent, possesses the same creative abilities on par with the human creators it was trained on. 
Therefore, the debate on whether AI can be as creative as human is reduced to the question of its ability to fit a sufficient amount of conditional data without marginalizing generation conditions.
This analysis offers practical guidelines, emphasizing the significance of gathering data on generative conditions and integrating these conditions into the training of models.
In conclusion, this work serves as a theoretical framework for understanding AI creativity. We have not only theoretically answered whether AI can be as creative as humans but also provided practical methodologies for assessing and fostering AI's creative potential.


\vspace{10pt}
\noindent\textbf{Limitations.} 
This paper reduces the question of whether AI can be as creative as human into a question of its ability to fit a massive amount of data. The challenge lies in the availability of sufficient conditional data satisfying the desired level of creativity and the capacity of the AI model to fit those data. This paper does not address whether AI can effectively fit the necessary data. There exists a possibility that AI's limitation in learning the required data volume could hinder its achievement of creativity. We would like to leave the discussion of it in future work. 
Besides, this paper conceptualize relative creativity and proposing a methodology to assess the extent to which an AI model achieves this type of creativity. The development of a comprehensive dataset and the establishment of a benchmark for creativity evaluation are forthcoming endeavors. Additionally, we will test current advanced AI models to evaluate their creative performance. As we progress, we anticipate integrating our methods into the existing evaluation benchmark toolkit, align with the evolving landscape of AI technologies and assessment.


\bibliography{all}
\bibliographystyle{apalike}

\newpage
\appendices
\onecolumn
\section{Proofs} 
\label{app:proof}
\subsection{Proof of Theorem \ref{thm:1}}
\begin{proof}
Since $L(q(\cdot\mid \I),c) \ge 0$, by using Markov's inequality, it holds  that with probability at least $1-\delta$ over the draw of   $\bc \sim \Dcal$,
$$
L(q(\cdot\mid \bI),\bc)< \delta^{-1}\EE_{c \sim \Dcal_{C}}[L(q(\cdot\mid \I),c)].
$$ 
Since $L(q(\cdot\mid \I),c) \in \{0,1\} \subset [0,1]$, via Hoeffding's inequality \citep{hoeffding1963probability}, it holds that with probability at least $1-t$ over the draw of   $(c_i)_{i=1}^n \sim (\Dcal_{C})^{\otimes n}$,
\begin{align*}
\EE_{c \sim \Dcal_{C}}[L(q(\cdot\mid \I),c)] &<\frac{1}{n}\sum_{i=1}^n L(q(\cdot\mid \Ii),c_i)+ \sqrt{\frac{\ln(1/t)}{2n}}
\end{align*}
Thus, with probability at least $1-t$ over the draw of   $(c_i)_{i=1}^n \sim (\Dcal_{C})^{\otimes n}$,
it holds that with probability at least $1-\delta$ over the draw of   $\bc \sim \Dcal$,
$$
L(q(\cdot\mid \bI),\bc) <\delta^{-1}E_{0}+\delta^{-1}  \sqrt{\frac{\ln(1/t)}{2n}},
$$
where $E_{0}=\frac{1}{n}\sum_{i=1}^n L(q(\cdot\mid \Ii),c_i)$. Here, we have that if  $n \ge\frac{\ln(1/t)}{2(\delta-E_{0})^{2}} $ and $\delta-E_{0} > 0$,
$$
\delta^{-1}E_{0}+\delta^{-1}  \sqrt{\frac{\ln(1/t)}{2n}} \le 1.
$$  Thus,  if  $n \ge\frac{\ln(1/t)}{2(\delta-E_{0})^{2}} $ and $\delta-E_{0} > 0$,
with probability at least $1-t$ over the draw of   $(c_i)_{i=1}^n \sim (\Dcal_{C})^{\otimes n}$, it holds that with probability at least $1-\delta$ over the draw of   $\bc \sim \Dcal$,
$$
L(q(\cdot\mid \bI),\bc) < 1.
$$
Since $L(q(\cdot\mid \bI),\bc)  \in \{0,1\}$, then $L(q(\cdot\mid \bI),\bc)=0$. \end{proof}

\subsection{Proof of Theorem \ref{thm:2}}
\begin{proof}From Assumption \ref{assumption:1},
\begin{align*}
\EE_{c\sim \Dcal_C}[L(q(\cdot\mid \I),c)] &\le\EE_{c\sim \Dcal_C}[\one\{\kl{p(\cdot\mid c)}{q(\cdot\mid \I)} \ge  \tau \}].  
\end{align*}
We bound the right-hand side of this inequality by using a function similar to hinge loss with a slope as follow: for any $r(c)>0$, 
\begin{align*}
\EE_{c\sim \Dcal_C}[L(q(\cdot\mid \I),c)] & =\EE_{c\sim \Dcal_C}[\one\{\tau-\kl{p(\cdot\mid c)}{q(\cdot\mid \I)}\le0   \}] \\ &\le \EE_{c\sim \Dcal_C}\left[\max\left(0,1-\frac{1}{r(c)}\left(\tau-\kl{p(\cdot\mid c)}{q(\cdot\mid \I)}\right) \right)\right].
\end{align*}
Here, from the definition of the KL divergence,   
$$
\kl{p(\cdot\mid c)}{q(\cdot\mid \I)}= \EE_{x \sim p(\cdot\mid c)}\left[\log \frac{p(x\mid c)}{q(x\mid I[c])}\right]= \EE_{x \sim p(\cdot\mid c)}\left[\log p(x\mid c)\right] - \EE_{x \sim p(\cdot\mid c)}[\log q(x\mid I[c])].
$$
Substituting this into the above inequality, 
\begin{align*}
\EE_{c\sim \Dcal_C}[L(q(\cdot\mid \I),c)] & \le \EE_{c\sim \Dcal_C}\left[\max\left(0,1-\frac{\tau}{r(c)}+\frac{\kl{p(\cdot\mid c)}{q(\cdot\mid \I)}}{r(c)}  \right)\right]
 \\ & =\EE_{c\sim \Dcal_C}\left[\max\left(0,A+B  \right)\right],
\end{align*}
where $A=1-\frac{\tau-\EE_{x \sim p(\cdot\mid c)}\left[\log p(x\mid c)\right]}{r(c)}$ and $B=-\frac{ \EE_{x \sim p(\cdot\mid c)}[\log q(x\mid I[c])]}{r(c)}$. Here, we notice that $B \ge 0$ because $r(c)>0$ and $-\log q(x\mid I[c]) \ge 0$ (here we use the fact that $q(\cdot\mid I[c])$ is the probability mass function instead of density). By using the fact of   $B \ge 0$  and the definition of maximum operator, we expand  the cases
of different values of $A$ and $B$ as
\begin{align*}
\max\left(0,A+B  \right) & = \begin{cases}A+B & \text{if $A+B \ge 0$ } \\
0 & \text{if $A+B < 0$} \\
\end{cases} 
 \\ & = \begin{cases}A+B & \text{if $A \ge 0$ } \\
A+B & \text{if $-B\le A < 0$} \\
0  & \text{if $A <-B$}
\end{cases} 
 \\ & \le  \begin{cases}A+B & \text{if $A \ge 0$ } \\
B & \text{if $-B\le A < 0$} \\
B  & \text{if $A <-B$}
\end{cases} 
\end{align*}
where the last line follows from the fact that $A+B \le B$ if $-B\le A < 0$ and $0\le B$.
Since the last two cases output the same value $B$ whenever $A<0$, we can simplify this expression as
\begin{align*}
\max\left(0,A+B  \right) & \le  \begin{cases}A+B & \text{if $A \ge 0$ } \\
B & \text{if $A < 0$} \\
\end{cases}
\\ & = \max(0,A)+B. 
\end{align*}
Substituting this into the above inequality of $\EE_{c\sim \Dcal_C}[L(q(\cdot\mid \I),c)]\le\EE_{c\sim \Dcal_C}\left[\max\left(0,A+B  \right)\right]$,
\begin{align*}
\EE_{c\sim \Dcal_C}[L(q(\cdot\mid \I),c)] & \le\EE_{c\sim \Dcal_C}[\max(0,A)]+\EE_{c\sim \Dcal_C}[B]. 
\end{align*}
We now set $r(c)>0$ to remove the 1st term on the right-hand side of this inequality as 
\begin{align*}
A \le 0 & \Longleftrightarrow\ 1-\frac{\tau-\EE_{x \sim p(\cdot\mid c)}\left[\log p(x\mid c)\right]}{r(c)} \le 0
\\ & \Longleftrightarrow\ r(c) \le\tau-\EE_{x \sim p(\cdot\mid c)}\left[\log p(x\mid c)\right].
\end{align*}
Here, notice that $-\EE_{x \sim p(\cdot\mid c)}\left[\log p(x\mid c)\right]=H[ p(\cdot\mid c)]$, which is the entropy of $ p(\cdot\mid c)$. Thus, by setting $r(c) =\tau+H[ p(\cdot\mid c)]$, we have that $\EE_{c\sim \Dcal_C}[\max(0,A)]=0$ and thus, 
\begin{align*}
\EE_{c\sim \Dcal_C}[L(q(\cdot\mid \I),c)] \le\EE_{c\sim \Dcal_C}[B]&=\EE_{c\sim\Dcal_C}\left[-\frac{1}{r(c)} \EE_{x \sim p(\cdot\mid c)}[\log q(x\mid I[c])]\right] \\ & =\EE_{c\sim\Dcal_C}\EE_{x \sim p(\cdot\mid c)}\left[-\frac{1}{r(c)} \log q(x\mid I[c]) \right]
\\ & =\EE_{(x,c)\sim\Dcal}\left[-\frac{1}{r(c)} \log q(x\mid I[c]) \right], 
\end{align*}
where $\Dcal(x,c)=  \Dcal_C(c)p(x\mid c)$. Here, since  $-\log q(x\mid I[c])\in [0, M]$ and $\frac{1}{r(c)} \in [0,\frac{1}{\rmin}]$ almost surely, we have that $-\frac{1}{r(c)} \log q(x\mid I[c]) \in[0,\frac{M}{\rmin}]$ almost surely. (The discussion regarding the upper bound of the negative log likelihood, $-\log q$, is detailed in Appendix~\ref{apd:upper_nll}.)
Thus, by using Hoeffding's inequality \citep{hoeffding1963probability}, it holds that with probability at least $1-t$ over the draw of   $(x_{i},c_i)_{i=1}^n \sim \Dcal^{\otimes n}$,
\begin{align*}
\EE_{(x,c)\sim\Dcal}\left[-\frac{1}{r(c)} \log q(x\mid I[c]) \right]<-\frac{1}{n}\sum_{i=1}^n \frac{1}{r(c_{i})} \log q(x_{i} \mid I[c_{i}])+\frac{M}{\rmin} \sqrt{\frac{\ln(1/t)}{2n}}.
\end{align*} 
Combining with above two inequalities, 
$$
\EE_{c\sim \Dcal_C}[L(q(\cdot\mid \I),c)]   <-\frac{1}{n}\sum_{i=1}^n \frac{1}{r(c_{i})} \log q(x_{i} \mid I[c_{i}])+\frac{M}{\rmin} \sqrt{\frac{\ln(1/t)}{2n}}. 
$$
Since the log of the products is the sum of logs, using  $q (x \mid  I[c])=\prod_{t=1}^T q(x^{(t)} \mid x^{(t-\omega:t-1)},I[c])$,
\begin{align*}
\EE_{c\sim \Dcal_C}[L(q(\cdot\mid \I),c)] <-\frac{1}{n}\sum_{i=1}^n \frac{1}{r(c_{i})} \sum_{t=1}^T \log q(x_{i} ^{(t)}\mid x^{(t-\omega:t-1)}_{i},I[c_{i}])+\frac{M}{\rmin} \sqrt{\frac{\ln(1/t)}{2n}}.
\end{align*}

Since $L(q(\cdot\mid \I),c) \ge 0$, by using Markov's inequality, it holds  that with probability at least $1-\delta$ over the draw of   $\bc \sim \Dcal$,
$L(q(\cdot\mid \bI),\bc)< \delta^{-1}\EE_{c \sim \Dcal_{C}}[L(q(\cdot\mid \I),c)]$. Thus, with probability at least $1-t$ over the draw of   $(x_{i},c_i)_{i=1}^n \sim \Dcal^{\otimes n}$,
it holds that with probability at least $1-\delta$ over the draw of   $\bc \sim \Dcal$,
$$
L(q(\cdot\mid \bI),\bc) <\delta^{-1} E_1+\frac{\delta^{-1}M}{\rmin} \sqrt{\frac{\ln(1/t)}{2n}}. 
$$   
where $E_1=-\frac{1}{n}\sum_{i=1}^n \frac{1}{r(c_{i})} \sum_{t=1}^T \log q(x_{i} ^{(t)}\mid x^{(t-\omega:t-1)}_{i},I[c_{i}])$. Here, we have that if  $n \ge\frac{M^{2}\ln(1/t)}{2\rmin^{2}(\delta-E_{1})^{2}} $ and $\delta-E_{1} > 0$,
$$
\delta^{-1} E_1+\frac{\delta^{-1}M}{\rmin} \sqrt{\frac{\ln(1/t)}{2n}}\le 1.
$$  Thus,  if  $n \ge\frac{M^{2}\ln(1/t)}{2\rmin^{2}(\delta-E_{1})^{2}} $ and $\delta-E_{1} > 0$,
with probability at least $1-t$ over the draw of   $(x_{i},c_i)_{i=1}^n \sim \Dcal^{\otimes n}$, it holds that with probability at least $1-\delta$ over the draw of   $\bc \sim \Dcal$,
$$
L(q(\cdot\mid \bI),\bc) < 1.
$$
Since $L(q(\cdot\mid \bI),\bc)  \in \{0,1\}$, then $L(q(\cdot\mid \bI),\bc)=0$. 
 
\end{proof}

\subsection{Proof of Corollary \ref{coro:2}}
\begin{proof}
Following all the proof steps of Theorem \ref{thm:1} while replacing $c$ and $\Dcal_C$ by $z=(u,c)$ and $\Dcal_{U,C}$, we have the following. If  $E_{2}<\delta$ and   $n \ge\frac{\ln(1/t)}{2(\delta-E_{2})^{2}} $, with probability at least $1-t$ over the draw of    $(u_{i}, c_i)_{i=1}^n \sim (\Dcal_{U,C})^{\otimes n}$, it holds that with probability at least $1-\delta$ over the draw of  $(\bu ,\bc )\sim \Dcal_{U,C}$, 
\begin{align*}
L(\Mcal(\bz), \bz) &< \delta^{-1}\EE_{z\sim \Dcal_{U,C}}[L(\Mcal(z),z)] 
 <\frac{\delta^{-1}}{n}\sum_{i=1}^n L(\Mcal(z_{i}),z_{i})+\delta^{-1} \sqrt{\frac{\ln(1/t)}{2n}}  \le 1.
\end{align*}
This implies the desired statement. 
\end{proof}

\subsection{Proof of Corollary \ref{coro:3}}
\begin{proof}
Following all the proof steps of Theorem \ref{thm:2} while replacing $z$ and $\Dcal_C$ by $z=(u,c)$ and $\Dcal_{U,C}$, we have the following. By setting $r(z) =\tau+H[ p(\cdot\mid z)]$, if  $E_{3}<\delta$ and   $n \ge\frac{M^{2}\ln(1/t)}{2\rmin^{2}(\delta-E_{3})^{2}} $, with probability at least $1-t$ over the draw of    $(x_{i},u_{i}, c_i)_{i=1}^n \sim \Dcal_{}^{\otimes n}$, it holds that with probability at least $1-\delta$ over the draw of  $(\bu ,\bc )\sim \Dcal_{U,C}$, 
\begin{align*}
L(\Mcal(\bz), \bz) &<\delta^{-1}\EE_{z \sim \Dcal_{U,C}}[L(\Mcal(z),z)] 
\\  &\le\delta^{-1} \EE_{z \sim \Dcal_{U,C}}\left[\max\left(0,1-\frac{1}{r(z)}\left(\tau-\kl{p(\cdot\mid z)}{\Mcal(z)}\right) \right)\right]
\\ & \le\delta^{-1}\EE_{(x,u,c)\sim\Dcal}\left[-\frac{1}{r(z)} \log q(x\mid u,I[c]) \right],
\\ & \le -\frac{\delta^{-1}}{n}\sum_{i=1}^n \frac{1}{r(z_{i})} \log q(x_{i} \mid u_{i},I[c_{i}])+\frac{\delta^{-1}M}{\rmin} \sqrt{\frac{\ln(1/t)}{2n}} \\ & =-\frac{\delta^{-1}}{n}\sum_{i=1}^n \frac{1}{r(z_{i})} \sum_{t=1}^T \log q(x_{i} ^{(t)}\mid x^{(t-\omega:t-1)}_{i},u_{i}, I[c_{i}])+\frac{\delta^{-1}M}{\rmin} \sqrt{\frac{\ln(1/t)}{2n}}\le 1.
\end{align*} 
This implies the desired statement. 
\end{proof}

\subsection{Proof of Corollary \ref{coro:4}}
\begin{proof}
Following all the proof steps of Theorem \ref{thm:2} while replacing $z$ and $\Dcal_C$ by $z=(u,c)$ and $\Dcal_{U,C}$, we have the following. Since $r(z) =\tau+H[ p(\cdot\mid z)]$, we have that with probability at least $1-\delta$ over the draw of   $\bz=(\bu ,\bc )\sim \Dcal_{U,C}$, 
\begin{align*}
L_{\bar z}[q_{S}]&<\delta^{-1}\EE_{(u,c) \sim \Dcal_{U,C}}[L( q_{S}(\cdot \mid u,I[c]),u,c)] 
\\  &\le\delta^{-1} \EE_{(u,c) \sim \Dcal_{U,C}}\left[\max\left(0,1-\frac{1}{r(u,c)}\left(\tau-\kl{p(\cdot\mid u,c)}{ q_{S}(\cdot\mid u,I[c])}\right) \right)\right]
\\ & \le\delta^{-1}\EE_{(x,u,c)\sim\Dcal}\left[-\frac{1}{r(z)} \log q_{S}(x\mid u,I[c]) \right].
\end{align*} 
By using the definition of $Q(t)$ and taking the union bound, we have that with probability at least $1-\delta$ over  an draw of $S=(v_{i})_{i=1}^n \sim \Dcal^{\otimes n}$ and $\bz\sim \Dcal_{U,C}$,  
\begin{align*}
L_{\bar z}[q_{S}] &<-\frac{2\delta^{-1}}{n}\sum_{i=1}^n \frac{1}{r(z_{i})} \log q_{S}(x_{i} \mid u_{i},I[c_{i}])+\Ocal\left( \sqrt{\frac{Q(\delta/2)}{\delta ^{2}n}}\right)
\\ & =-\frac{2\delta^{-1}}{n}\sum_{i=1}^n \frac{1}{r(z_{i})} \sum_{t=1}^T \log q_{S}(x_{i} ^{(t)}\mid x^{(t-\omega:t-1)}_{i},u_{i}, I[c_{i}])+\Ocal\left( \sqrt{\frac{Q(\delta/2)}{\delta ^{2}n}}\right).
\end{align*}
This proves the first inequality in the desired statement. The second inequality in the desired statement follows from the fact that  $\frac{1}{r(z_{i})} \le \frac{1}{\rmin}$ almost surely. 
\end{proof}

\section{Additional Discussions}
\subsection{On the use of context prompting}
The concept of relative creativity, as defined in Definition \ref{def:1}, does not explicitly incorporate the conditioning on a specific context prompt. Nevertheless, all derived results seamlessly adapt to such context conditioning. This adaptation is achieved by substituting the original probability space with a conditional probability space, predicated on a given context. Given that a conditional probability space retains the fundamental properties of a probability space, the application of this theory to the conditional probability scenario necessitates no alterations. Consequently, in Sections \ref{sec:stat_relative_creativity_autoregressive} and \ref{sec:stat_relative_creativity_prompt}, we expand the statistical creativity framework to encompass the context prompting scenario. The findings, as presented in Corollary~\ref{coro:2} and \ref{coro:3}, evaluate statistical creativity within the setting of model performance under context prompting.


\subsection{The upper bound on the negative log likelihood}
\label{apd:upper_nll}
In Theorem  \ref{thm:2}, the condition of $-\log q(x\mid I[c])\le M$ is easily satisfiable, e.g., by using softmax over choices or  any distribution that puts non-zero probability over all $x \in \Xcal$. Even when we have $q$ that does not satisfy this condition,   we can easily modify $q$ to $q+\epsilon$ with a normalization for some small $\epsilon>0$: since a normalized version of $q+\epsilon$\  puts non-zero probability over all $x \in \Xcal$, this satisfies the condition.

\end{document}